\pdfoutput=1

\documentclass[11pt]{article}

\usepackage{acl}

\usepackage{times}
\usepackage{latexsym}

\usepackage[T1]{fontenc}

\usepackage[utf8]{inputenc}

\usepackage{microtype}
\usepackage{booktabs}
\usepackage{amsmath}
\usepackage{amsthm}
\usepackage{amssymb} 
\usepackage{graphicx}
\usepackage{multirow}
\newtheorem{proposition}{Proposition}
\title{Incorporating Exponential Smoothing into MLP:\\ A Simple but Effective Sequence Model}

\author{Jiqun Chu \\
	Peking University  \\
	\texttt{chujiqun@pku.edu.cn} \\\And
	Zuoquan Lin \footnote{corresponding author}\\
	Peking University  \\
	\texttt{linzuoquan@pku.edu.cn} }

\begin{document}
\maketitle

\begin{abstract}	
Modeling long-range dependencies in sequential data is a crucial step in sequence learning.  A recently developed model, the Structured State Space (S4),  demonstrated significant effectiveness in modeling long-range sequences. However, It is unclear whether the success of S4 can be attributed to its intricate parameterization and HiPPO initialization or simply due to State Space Models (SSMs). To further investigate the potential of the deep SSMs, we start with exponential smoothing (ETS), a simple SSM, and propose a stacked architecture by directly incorporating it into an element-wise MLP. We augment simple ETS with additional parameters and complex field to reduce the inductive bias. Despite increasing less than 1\% of parameters of element-wise MLP, our models achieve comparable results to S4 on the LRA benchmark.\footnote{Our codes and scripts are available at \url{https://github.com/PKUAI-LINGroup/ETSMLP}.}
\end{abstract}

\section{Introduction}

Transformer \cite{10.5555/3295222.3295349} and its variants have been the most successful architecture in various domains of deep learning. However, the self-attention layer, which plays a crucial role in contextualizing the input, poses a significant computational and memory burden with a complexity of $O(L^2)$. This limitation hinders the application of the transformers in modeling long sequences, particularly when operating under hardware constraints, which is a common scenario for large language models.  To alleviate this issue, several models have been proposed to reduce the computational and memory requirements of the transformers  \cite{beltagy2020longformer,choromanski2020rethinking,kitaev2020reformer,wang2020linformer,guo2021longt5,kasai2021finetuning,peng2021random,dao2022flashattention,hua2022transformer,10.1145/3530811,10.1145/3586074,zandieh2023kdeformer}. Despite these efforts, all the models are only partial modifications of the attention mechanism and struggle to perform well on long-range sequence benchmarks such as Long Range Arena (LRA) \cite{tay2020long}.

In a recent breakthrough result, \cite{gu2021efficiently}  introduced a novel framework called the "structured state space sequence" (S4) that leveraged the State Space Models (SSMs). S4 builds upon continuous-time SSMs and addresses the computational bottleneck of previous approaches by introducing the Normal Plus Low-Rank (NPLR) decomposition of the state matrices. Additionally, the initialization of state matrices utilizes HiPPO matrices which have been demonstrated to be effective in sequence learning in \cite{gu2020hippo}. Notably, S4 exhibited exceptional performance across various sequential tasks, particularly in the LRA, where it outperformed the existing transformer variants by an impressive 
accuracy.

Despite the impressive performance of S4, its intricate parameterization and strict initialization schemes impede researchers from fully comprehending, implementing, and analyzing the model. Although there have been attempts to simplify the S4 framework by \cite{smith2022simplified,gupta2022diagonal}, these models still required the HiPPO initialization process. Other studies have explored the relationship between SSMs and recurrent units or global convolutions and demonstrated strong performance on various tasks \cite{li2022makes,orvieto2023resurrecting}. These works highlight the potential of SSMs and suggest that simpler yet effective SSM architectures may exist.

In our work, we deviate from the methodology proposed by S4, which begins with the continuous SSM and then simplifies the process. We initiate our approach with a discrete SSM, namely Exponential Smoothing (ETS), and introduce additional parameters to reduce the inductive bias.  This alternative approach offers two notable advantages. Firstly, it circumvents the simplification of the continuous SSMs that need sophisticated mathematical derivations and thus enhances accessibility and comprehensibility. Secondly, it explores the possibility of random initialization departing from HiPPO initialization for continuous SSMs. For the streamlining of the model, our architecture directly integrates a parameterized ETS into an element-wise Multi-Layer Perceptron (MLP). By incorporating less than 3\% of the total parameters after the initial linear layer of the MLP, we successfully transform a channel-only MLP into a sequence learner. 

We conducted experiments on multiple datasets, including the LRA and several Natural Language Understanding (NLU) datasets. Despite its simplicity, surprisingly, our model performs comparably to S4. In all six tasks in the LRA, our results slightly surpass the performance of S4 and DSS by 2.61 points on average and significantly outperform the transformer variants by about 20 points. In addition, we evaluated our model on seven NLU datasets and consistently achieved comparable performance with the transformer encoders. The findings of our work shed light on the potential of SSMs from a unique standpoint, where simply incorporating an ETS into an MLP can achieve a similar effect as the transformer model. A thorough examination of the proposed model was undertaken through an ablation study on the hyperparameters and an evaluation of the initialization method. Additional experiments were conducted to compare our model with the transformer model for efficiency and memory utilization, especially in handling lengthy texts. The results of these experiments provide evidence of the advantages of our model over the transformer model in terms of time and memory complexity.

In summary, our main contributions are as follows:
\begin{itemize}
    \item We introduce the Exponential Smoothing Multi-Layer Perceptron (ETSMLP) model. We integrate the enhanced ETS module into an element-wise MLP to create an effective sequence model.
    \item We evaluate ETSMLP on the LRA and conduct comparative experiments with transformer encoders on various NLU datasets. The empirical results demonstrate the effective capacity in long-range sequence modeling.
    \item We conduct ablation studies on the proposed parameters and initialization methods. Additionally, we emphasize the advantages of SSMs over the attention mechanism in speed and memory efficiency.
\end{itemize}

\section{Preliminaries}

We introduce basic notations and briefly review SSMs and ETS in this section. Focusing on time-invariant sequence models, we aim to transform a sequence of inputs $X=\{x_1,\dots,x_n\}\in \mathbb{R}^{n\times d}$ into a corresponding sequence of outputs $Y=\{y_1,\dots,y_n\}\in \mathbb{R}^{n\times d}$ with each output $y_i$ is exclusively based on historical data $x_1,\dots,x_i$. 

\begin{figure}[t]
	\centering	
	\includegraphics[width=.4\linewidth]{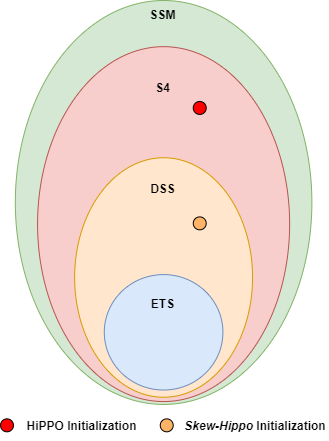}
	\caption{The relations among SSM, S4, DSS, and ETS. The HiPPO initialization is pointed out in red while the \textit{Skew-Hippo} initialization is pointed out in orange.}
	\label{Fig:ETS}	
\end{figure}

\subsection{State space models}

The continuous-time SSM is characterized by the differential equation (\ref{eq:diff}), which establishes the relationship between a continuous-time scalar input $x(t)$ to a scalar output $y(t)$ with the state matrix $A \in \mathbb{R}^{d \times d}$ and vectors $B\in \mathbb{R}^{d\times 1}$, $C\in \mathbb{R}^{1\times N}$:

\begin{equation}
\frac{dh}{dt}(t)=Ah(x)+Bx(t)\quad,\quad y(t)=Ch(t).
\label{eq:diff}
\end{equation}

\noindent If we set a sample time interval $\Delta >0$, and assume that the duration of sampling remains constant $\Delta$, we convert the continuous-time SSM into a discrete-time one using a recursive equation in the following:

\begin{equation}
h_k=\bar{A}h_{k-1}+\bar{B}x_k\quad,\quad y_k=\bar{C}h_k,
\end{equation}

\noindent where $\bar{A}=e^{A\Delta}$,$\bar{B}=(\bar{A}-I)A^{-1}B$ and $\bar{C}=C$. With $x_0=0$, we unroll this recursion explicitly as the equation (\ref{eq:unro}) which can be vectorized into a convolution in the equation (\ref{eq:conv}) with the SSM convolution kernel defined in the equation (\ref{eq:kernal}) as follows:

\begin{gather}
y_k=\sum_{j=0}^{k}\bar{C}\bar{A}^j\bar{B} x_{k-j}, \label{eq:unro}\\
\bar{K}\in \mathbb{R}^L=(\bar{C}\bar{B},\dots ,\bar{C}\bar{A}^{L-1}\bar{B}), \label{eq:kernal}\\
y=\bar{K}*x. \label{eq:conv}
\end{gather}

\noindent If we obtain the kernel $\bar{K}$, the convolution function aforementioned can be efficiently computed with Fast Fourier Transform (FFT) in $O(L\log (L))$ \cite{cormen2022introduction}. Nevertheless, the main challenge in the computation of SSMs is how to efficiently compute $\bar{K}$ from the matrices $\bar{A},\bar{A}$ and $\bar{C}$. S4 proposes an effective parameterization through decomposing matrix $\bar{A}$ to the NPLR matrices \cite{gu2021efficiently}, and diagonal state spaces (DSS) only consider the circumstances when $A$ is diagonalizable over $\mathbb{C}$ \cite{gupta2022diagonal}. Both methods involve intricate mathematics, sophisticated parameterization, and strict initialization, all of which are indispensable for achieving excellent performance. Our method will start from a special SSM, namely ETS, which gives a new insight into  
this problem, and requires fewer mathematical operations, fewer parameters, and more flexible initialization.

\subsection{Exponential smoothing}

ETS is a time series forecasting method that utilizes a weighted average of past observations to predict future values \cite{winters1960forecasting,hunter1986exponentially,hyndman2008forecasting}. The fundamental idea behind ETS is to give more weight to recent observations and less to older ones, with the weights decreasing exponentially as the observations get older. The core recursive equation for this method is the equation (\ref{eq:EST}) with the smoothing factor $\lambda$ in the range $(0,1)$:

\begin{equation}
y_t=\lambda x_t+(1-\lambda)y_{t-1}.
\label{eq:EST}
\end{equation}

ETS is a special SSM, with the substitution $\bar{A}=1-\lambda,\bar{B}=\lambda,\bar{C}=1$. However, compared with SSMs, 
ETS cannot capture sequential information effectively. Figure \ref{Fig:ETS} illustrates the relationship among 
SSMs, S4, DSS, and ETS. S4 and DSS are derived from the continuous-time SSMs with the difference that S4 decomposes the matrix $A$ into an NPLR matrix, while DSS assumes $A$ to be diagonalizable. As a result, the HiPPO initialization in S4 cannot directly adapt to DSS \cite{gu2020hippo}. The initialization in DSS is \textit{skew-Hippo} initialization which is the normal part of the HiPPO matrix. 

ETS serves as a special case within the realm of discrete-time SSM. In our approach, we incorporate parameters directly from ETS, distinguishing ours from S4 and DSS methods that simplify equations based on continuous-time SSMs.

\section{Exponential Smoothing Multi-layer Perceptrons}

In this section, we present our ETSMLP. We first introduce a complex exponential smoothing module which is the pivotal component of our architecture. We then describe the full architecture with two proposed versions, ESMLP and ESMLP-Gate.

\subsection{Complex exponential smoothing module}

\textbf{Learnable damped factors}. Damped factors are a commonly used technique of ETS for attenuating the influence of specific factors \cite{gardner1985exponential,mckenzie2010damped}. We introduce two learnable damped factors $\alpha$ and $\beta$ into simple ETS in equation (\ref{eq:ets}). The factor $\alpha$ controls the learning of $\lambda$ in an exponential scalar. A small $\alpha$ close to zero amplifies the impact of $\lambda$ and results in $\lambda^\alpha$ approximating 1 while a large $\alpha$ diminishes its impact, driving the combination closer to 0. The factor $\beta$ serves as a multiplicative factor that controls the influence of the current input $x_t$. The recursion can be unrolled in equation (\ref{eq:etsunrol}) with the kernel defined by the equation (\ref{eq:etsk}) as follows:

\begin{gather}
y_t=(1-\lambda^\alpha)\beta x_t+(\lambda^\alpha) y_{t-1}, \label{eq:ets}\\
y_t=\sum_{i=0}^{t}(\lambda^\alpha)^i(1-\lambda^\alpha)\beta x_{t-i}, \label{eq:etsunrol}\\
K=((1-\lambda^\alpha)\beta,\dots,(\lambda^\alpha)^{L-1}(1-\lambda^\alpha)\beta ). \label{eq:etsk}
\end{gather}

\noindent \textbf{Complex parameters}. Complex parameters in ETS have been demonstrated to capture both level and trend information in forecasting \cite{svetunkov2022complex}. By extending the learning capacity and enlarging the parameter space, the transformation from real to complex numbers is beneficial. Therefore, we treat $\alpha$, $\lambda$, $\beta$ as complex numbers, and keep the input $x_t$ and the output $y_t$ real. Consequently, only the real part of the kernel coefficients is utilized, and the corresponding computation formula is as follows:

\begin{equation}
y_t=\sum_{i=0}^{t}\Re((\lambda^\alpha)^i(1-\lambda^\alpha)\beta) x_{t-i}.
\label{eq:cest}
\end{equation}

\begin{figure*}[t]
	\centering	
	\begin{minipage}[t]{0.5\linewidth}
		\centering
		\includegraphics[width=\linewidth]{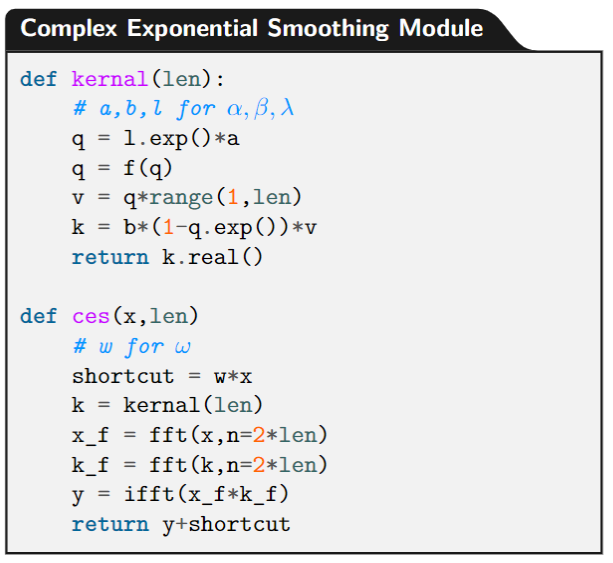}
	\end{minipage}
	\begin{minipage}[t]{0.35\linewidth}
		\centering
		\includegraphics[width=\linewidth]{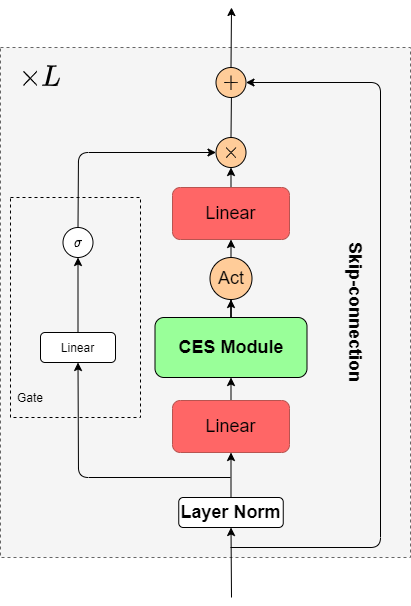}
	\end{minipage}
	\caption{Overview of the ETSMLP architecture. The left is the pseudo-code of the complex exponential smoothing (CES) module. The right is the entire architecture with a gate mechanism. } 
	\label{Fig:Architecture}	
\end{figure*}

\noindent \textbf{Exponential parameterization}. Directly training the parameters $\lambda$ is infeasible, due to the rapid explosion of the gradient $\lambda$ within a few steps. This challenge becomes evident upon inspecting the equation (\ref{eq:EST}). Given the gradients of $y_t$ $\frac{\mathrm{d} L}{\mathrm{d} y_t}$, the gradient of $\lambda$ $\frac{\mathrm{d} L}{\mathrm{d} \lambda}$ could be derived from the formulas (\ref{eq:explode}). This reveals that the gradients of $\lambda$ are proportional to  $\frac1{1-\lambda}$. Consequently, as $\lambda$ approaches 1, the gradients of $\lambda$ will explode.

{\small
\begin{gather}
	\frac{\mathrm{d} L}{\mathrm{d} \lambda} =\sum_{t=1}^{N}  \frac{\mathrm{d} L}{d y_t} \sum_{j=1}^{t}[\frac{(t-j)}{\lambda}-(t-j+1)]\lambda^{t-j}x_j \nonumber\\
	\approx \sum_{t=1}^{N}  \frac{\mathrm{d} L}{\mathrm{d} y_t} \sum_{j=1}^{t}(-1*\lambda^{t-j}x_j)=-\frac{\sum_{t=1}^{N} (\frac{\mathrm{d} L}{\mathrm{d} y_t} y_t)}{1-\lambda} \label{eq:explode}
\end{gather}}

To address this issue, we propose an exponential transformation of the parameters. We trains the parameter $\lambda'=\log\log \lambda$ instead of the $\lambda$. We prove that the stability of learning $\lambda'$ constrains the gradients of $\lambda'$ within a specified range, as described in Proposition 1.

\begin{proposition}
Let $\lambda \in \mathbb{C}$  be within the interior of the hollow unit disc $D^\circ(0,1)=\{z||z|<1\}/\{(0,0)\}$. We define $\lambda'=\log\log \lambda$ which substitutes $\lambda$ in the equation (\ref{eq:cest}). If the gradients of $y_t$ satisfy $\sum_{0}^{L} \frac{\mathrm{d} L}{\mathrm{d} y_t} y_t < \infty$, the gradients of the real and imaginary parts of $\lambda'$ are bounded for all $\lambda \in D^\circ(0,1)$.
\end{proposition}

\noindent The proof of Proposition 1 is elementary and is provided in the appendix \ref{app:proof}. This proposition proves that the exponential parameterization $\lambda'$ gradients are stable in the feasible region of $\lambda$. 

\noindent \textbf{Constraint function and shortcut}. In addition to the settings as aforementioned, $\lambda^\alpha$ must lie within the feasible field $D^\circ(0,1)$. To address this, we introduce a constraint function to enforce the validity of these parameters, which can be formulated in equation (\ref{eq:cons}): 

\begin{equation}
f(\lambda)=\begin{cases}
\lambda,  & \mbox{if }|\lambda|<max_\lambda ; \\
\frac{max_\lambda}{|\lambda|}\lambda, & \mbox{if }|\lambda|\ge max_\lambda .
\end{cases}
\label{eq:cons}
\end{equation}

Although this solution is simple, it yields remarkably effective results. We also explored an alternative approach inspired by the separation of real and imaginary parts, as discussed in  \cite{gu2022parameterization,orvieto2023resurrecting}. Unfortunately, its performance is unsatisfactory, because the gradients of imaginary parts appear unstable and may explode in a few steps.

Moreover, we introduce a parameter $\omega$ to establish a shortcut from input to output, a commonly used technique in deep learning. This parameter serves as a gating unit that regulates the incoming input. The final output of our model can be described with the sigmoid function $\sigma$ as follows:

\begin{equation}
o=\sigma(\omega)x+y.
\end{equation}

\noindent \textbf{Bidirectional}. We describe a bidirectional model incorporating backward recursion as $y_t=(1-\lambda_2) x_t+\lambda_2 y_{t+1}$. By employing a bidirectional model, the influence of tokens is determined by both preceding and succeeding tokens, resulting in a wide-ranging receptive field. The kernel function is formed by combining the forward and backward kernels in equation (\ref{eq:bikernal}). We employ the circular convolution to compute the output with the input being zero-padded on the right side to twice its length.

\begin{align}
K&=(1-\lambda_1,\dots,(1-\lambda_1)\lambda_1^{L-1}, \nonumber \\
&(1-\lambda_2)\lambda_2^{L-1},\dots,1-\lambda_2).
\label{eq:bikernal}
\end{align}

A sketch of the full computation is presented on the left of Figure \ref{Fig:Architecture}. Initially, we calculate the kernel corresponding to the sequence length and subsequently apply FFT to compute the convolution of the inputs and the kernel.  The Complex Exponential Smoothing (CES) module produces the final results by combining the shortcut and the convolution outputs. Although the current code is designed for a unidirectional kernel, a bidirectional kernel can be easily achieved by connecting two unidirectional kernels using the equations as aforementioned.

\begin{table*}
	\centering
	{
		\begin{tabular}{|c|c|c|c|c|c|c|c|}
			\toprule
			models &\textit{ ListOps}& \textit{Text}& \textit{Retrieval} & \textit{Image} & \textit{Pathfinder} & \textit{Path-X} & Avg.\\
			\midrule
			Transformer &36.37&64.27&57.46&42.44&71.40&-&53.66\\
			Reformer& 37.27&56.10&53.40&38.07&68.50&-&50.36\\
			Linear Trans&16.13&65.90&53.09&42.34&75.30&-&50.46\\
			Performer& 18.01&65.40&53.82&42.77&77.05&-&51.18\\
			\midrule
			S4&58.35&76.02&87.09&87.26&86.05&88.10&80.48\\
			$DSS_{SOFTMAX}$&60.6&84.8&87.8&85.7&84.6&87.8&81.88\\
			$DSS_{EXP}$&59.7&84.6&87.6&84.9&84.7&85.6&81.18\\
			S5&62.15&89.02&\underline{91.4}&\underline{88.0}&\underline{95.33}&\underline{98.58}&\underline{87.46}\\
			Mega-chunk&58.76&\underline{90.19}&90.97&85.80&94.41&93.81&85.66\\
			\midrule
			ETSMLP&61.35&87.2&85.78&78.14&85.86&87.21&80.92\\
			ETSMLP-Gate&\underline{\textbf{62.55}}&\textbf{88.49}&86.72&75.34&\textbf{91.66}&\textbf{93.78}&83.09\\
			\bottomrule
	\end{tabular}}
	\caption{Performance on the LRA benchmark tasks. We follow the procedure reported in \cite{ma2022mega}, and report means across three seeds for our methods. The Bold scores indicate the best performance between S4, DSS, and our models. We also include and underline the state-of-art results of concurrent methods such as Mega and S5.}
	\label{tab:LRA}
\end{table*}	

\subsection{ETSMLP blocks}

We incorporate the CES module into the element-wise MLP to learn token-level information. The CES module facilitates the mix of input information at the token level, resulting in a mixed output containing sequence information. We integrated the CES module just before the activation function into the MLP in the full architecture, depicted in Figure \ref{Fig:Architecture}. The functions are described as follows:

\begin{align*}
	X=\operatorname{LayerNorm}(X_i)&\in \mathbb{R}^{L\times d},\\
	H=W_1X& \in \mathbb{R}^{L\times D},\\
	Y=\sigma(\operatorname{CES}(H))&\in \mathbb{R}^{L\times D},\\
	Z=W_2Y& \in \mathbb{R}^{L\times d}.
\end{align*}

Compared to standard MLP, the increased parameters constitute only $\frac{3}{d}$ of the original MLP, where $d$ is the embedding dimension. For a typical model with $d=512$, a modest increase 0.58\% parameters enables channel-only MLP to learn sequence information, which is previously unattainable. Moreover, the computational and memory complexity is lower than that of the self-attention, as detailed in Section \ref{sec:eff}

\noindent \textbf{Gated architecture}. To further enhance the expressive capacity of our model, we add a gate mechanism like \cite{cho-etal-2014-properties, shazeer2020glu, hua2022transformer}. This gate unit controls the output of each block. After obtaining the representation after layernorm, we pass it through a linear layer, derive the score using the sigmoid function, multiply it with the output from the preceding module, and obtain the output of our layer through a residual connection.  As in Figure \ref{Fig:Architecture}, we express these processes as follows:

\begin{align*}
G=\operatorname{sigmoid}(W_gX)&\in \mathbb{R}^{L\times d},\\
O=G\otimes Z&\in \mathbb{R}^{L\times d},\\
X_{i+1}=X_{i}+O&\in \mathbb{R}^{L\times d}.\\
\end{align*}

\section{Experiments}

We present an empirical comparison between our ETSMLP and other baseline models. Our experiments encompass a set of sequence modeling tasks, including LRA, MNLI, IMDB, etc. The main experiment results are divided into two subsections: LRA and NLU benchmarks. Furthermore, we conduct an ablation study to examine the influence of hyperparameters. Additional information about the experimental details and datasets can be found in Appendix \ref{app:datasets}.

\begin{table*}[t]
	\centering
	{
		\begin{tabular}{c|c|c|c|c|c|c|c}
			\toprule
			\multirow{2}{*}{models}&\multicolumn{3}{c|}{Classification}& \multicolumn{2}{c|}{Similarity}& \multicolumn{2}{c}{Inference}.\\
			&CoLA&SST-2&IMDB&QQP&MRPC&MNLI&QNLI\\
			\midrule
			transformer&69.2&\textbf{81.7}&\textbf{88.2}&80.6&71.1&58.7&61.2\\
			ESMLP&69.2&81.3&87.1&81.6&\textbf{71.6}&60.6&64.5\\
			ESMLP-Gate&\textbf{69.3}&81.2&87.1&\textbf{82.3}&70.3&\textbf{61.3}&\textbf{64.8}\\
			\bottomrule
	\end{tabular}}
	\caption{Performance on the several NLU tasks. We report accuracy scores averaged across three seeds for all the datasets. All models are trained from scratch and are of a fairly similar size. The bold scores indicate the highest performance of each dataset.}
	\label{tab:nlu}
\end{table*}

\subsection{LRA} 
\label{exp:LRA}

The LRA benchmarks are a collection of long sequence modeling tasks ranging from 1024 to over 16000 \cite{tay2020long}. In Table \ref{tab:LRA}, we compare our models to several variants of SSM 
and Transformer. We observe that our model outperforms all the Transformer variants and achieves the comparable performance of S4 on average which is 83.09 vs 80.48. Although we don't gain the highest average scores among all concurrent works, it still produces comparable results without relying on the attention in MEGA \cite{ma2022mega}, or Hippo initialization in S5 \cite{smith2022simplified}. When comparing the individual tasks horizontally, we observe that our model performs significantly better in text tasks such as \textit{ListOps} and \textit{Text}, while slightly underperforming on image tasks like \textit{Image}.  This discrepancy may be attributed to the weight decaying exponentially with distance, which is unsuitable for flattened patches.

We provide the hyperparameters used in our experiments in Appendix \ref{app:datasets}. 

\subsection{NLU}
\label{exp:NLU}

The LRA results demonstrate the benefits of our method in sequential text tasks. Furthermore, we conduct experiments on various NLU tasks and compare our models with a transformer encoder architecture trained from scratch. Our experimental evaluations were divided into three categories: sentence classification, including \textit{CoLA} \cite{warstadt-etal-2019-neural}, \textit{SST-2} \cite{socher-etal-2013-recursive}, and \textit{IMDB} \cite{maas-etal-2011-learning}; sentence similarity, including \textit{QQP}, \footnote{\url{https://quoradata.quora.com/First-Quora-Dataset-Release-Question-Pairs}} and \textit{MRPC} \cite{dolan-brockett-2005-automatically}; and natural language inference, including \textit{MNLI} \cite{williams-etal-2018-broad} and \textit{QNLI} \cite{rocktaschel2015reasoning}. We present the experiment results in Table \ref{tab:nlu}, which reveal that our architecture can achieve comparable or even superior performance to transformers on all the datasets. Considering the simple computation and slight increase in parameters on MLP, these results suggest that the ETS has tremendous potential in sequence learning tasks.

\subsection{Analysis}

\subsubsection{Role of damped factors and fields} 

The experimental results presented above are encouraging and demonstrate the effectiveness of the ETS for sequence modeling tasks. It is proved empirically that even the simplest SSM like ETS can achieve a competitive result compared with other state space model variants. To further consider if we would simplify the ETS in fewer parameters, we conducted ablation studies on the damped factors and the number fields. Table 3 shows the accuracy results of \textit{Listops} if we remove $\alpha$, $\beta$, or $\omega$ or change all parameters from complex to real fields. We could observe that whether to remove the $\alpha$ or $\beta$ or $\omega$ or the complex field, the performance of our method drops significantly, especially $\alpha,\omega$ and complex field. These experiments illustrate the necessity of our architecture in sequence modeling tasks.

\begin{table}[t]
	\centering
	\begin{tabular}{|c|c|c|}
		\toprule
		model & Arch & Acc \\
		\midrule
		\multirow{5}{*}{ETSMLP}&Real&40.5 \\
		&No $\alpha$&41.5 \\
		&No $\omega$&40.75\\
		&No $\beta$ &56.05 \\
		& - & 61.35\\
		\bottomrule
	\end{tabular}
	\caption{Ablation analysis of the learnable parameters $\alpha,\beta,\omega$ on \textit{ListOps}. "-" means the keeping all the parameters of our methods.}
	\label{tab:ab_df}
\end{table}	

\subsubsection{Initialization of parameters} 

S4 and its variants conducted several experiments on $HiPPO$ initialization and concluded that random initialization may hurt performance \cite{gu2021efficiently,gu2022parameterization}. Because of the different computation processes,  $HiPPO$ initialization doesn't work in our models. Therefore, we consider the ring initialization method, which involves uniform sampling on a ring defined by the range $\{r_{min}\le|z|\le r_{max}| z\in \mathbb{C }\}$. By predefining values for $r_{min}$ and $r_{max}$, we uniformly sample $\lambda$ along the ring, between circles with radii $r_{min}$ and $r_{max}$. In addition to examining the effects of different initializations, we conducted experiments using fixed-value initialization operations. Our experimental results on \textit{listops} are displayed in Figure \ref{Fig:init}. It can be observed that our model exhibits consistent performance across rings of varying sizes. However, when dealing with fixed points, the effectiveness diminishes significantly.

\begin{table}[t]
	\centering
	\begin{tabular}{|c|c|c|}
		\toprule
		model & Initialization & Acc \\
		\midrule
		\multirow{6}{*}{ETSMLP}&Stable(0.3,0)&52.4\\
		&Stable(0.7,0)&52.4\\
		&Stable(0.5,0)&52.95 \\
		&Ring(0.1,0.6)&60.55\\
		&Ring(0.6,0.9)&58.95 \\
		&Ring(0.1,0.9)&61.35 \\
		\bottomrule
	\end{tabular}
	\caption{ Performance of  \textit{ListOps} with different initialization. "Ring" means uniform sampling on a ring $\{r_{min}\le|z|\le r_{max}| z\in \mathbb{C }\}$. "Stable" means initialization $\lambda'$ on the same point.}
	\label{Fig:init}
\end{table}	

\begin{figure*}
	\centering	
	\begin{minipage}[t]{0.45\linewidth}
		\centering
		\includegraphics[width=.9\linewidth]{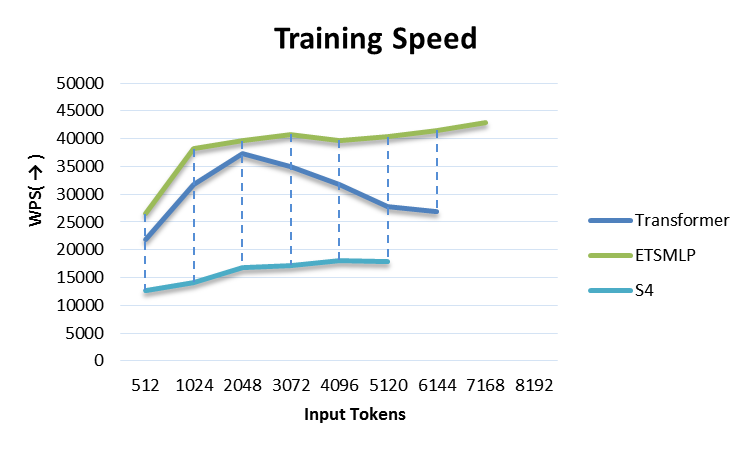}
	\end{minipage}
	\begin{minipage}[t]{0.45\linewidth}
		\centering
		\includegraphics[width=.9\linewidth]{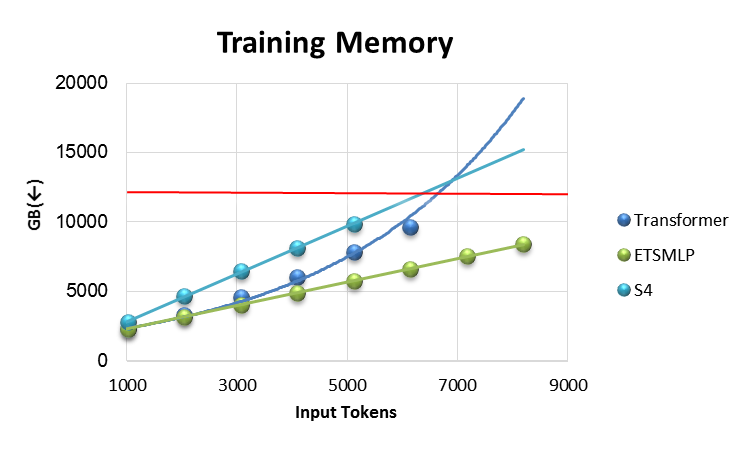}
	\end{minipage}
	\caption{A training speed and memory comparison between the transformer and ETSMLP. Both models have approximately 30M parameters, and the batch size remains constant at 1 under all circumstances.} 
	\label{Fig:Effi}	
\end{figure*}

\subsubsection{Efficiency and memory analysis}  
\label{sec:eff}

To assess the speed and memory efficiency across different lengths, we performed experiments using a synthetic dataset that combines multiple sentences to achieve sufficient length. Our chosen task is language modeling, as it allows us to segment sentences into desired lengths. The maximum length of our synthetic dataset is 8192. We adjusted the sample length within each batch to compare the words per second (WPS) and memory usage (in GB) between the transformer, S4, and our model at comparable sizes. The batch size was uniformly set to 1 to ensure accurate memory usage. All training procedures are carried out on an NVIDIA GeForce GTX 2080 GPU.

The comparison results are presented in Figure \ref{Fig:Effi}. Notice that our approach consistently achieves the highest WPS for all the sequence lengths. The slower performance of S4 can be attributed to its complex calculations on the NPLR. Both our model and S4 share a common characteristic: the WPS remains constant as the sequence length increases, while the transformer shows a decrease. Furthermore, the memory requirements of the transformer exhibit an almost quadratic growth, whereas our model and S4 demonstrate linear growth, with our model having a lower slope. For sequence lengths below 3072, there is minimal difference between our model and the transformer. However, as the training length increases, the undesirable quadratic growth in computation and memory complexity becomes apparent in the transformer, whereas our method avoids this issue.

\section{Related Works}

Since the Transformer was introduced, the quadratic time cost of the attention operation has been numerously researched. Optimizing this operation can improve the efficiency when training and inferencing long context for large language models \cite{xiao2023efficient}. Recently, many transformer variants have been introduced to reduce the complexity of attentions \cite{10.1145/3530811}, including sparse attention \cite{beltagy2020longformer,kitaev2020reformer,guo2021longt5}, kernel-based methods \cite{choromanski2020rethinking,kasai2021finetuning,peng2021random}, chunked attention with gating \cite{hua2022transformer,ma2022mega} and other efficient methods \cite{wang2020linformer,dao2022flashattention}. Another line of research tries to replace the attention mechanism with other modules for long sequences and avoid quadratic time costs. A dizzying number of attention-free models have emerged, where SSMs 
are becoming one of the most promising models among them.

\noindent \textbf{SSMs}. S4 first investigated the SSM for long sequence modeling \cite{gu2021efficiently}. They showed that naive instantiations of the SSM did not perform well but HiPPO-LegS matrix did \cite{gu2020hippo}, and hence introduced the DPLR that efficiently computed the complex diagonal plus low-rank matrix. DSS observed that a fully diagonal matrix could preserve the performance of the original S4 \cite{gupta2022diagonal}, and S4D \cite{gu2022parameterization} then showed that the initialization is critical for DSS. Inspired by S4, many SSM variants emerged recently. S5 replaced single-input, single-output (SISO) SSMs in S4 with multi-input, multi-output (MIMO) \cite{smith2022simplified}. SGConv viewed the SSM as a global convolution model and suggested that the convolution kernel's sub-linear decay in sequence length is indispensable \cite{li2022makes}. Linear Recurrent Unit (LRU) explored the relationship between the SSM and linear RNN and showed the importance of initialization, exponential parameterization, and normalization for SSMs \cite{orvieto2023resurrecting}. MEGA was the most similar approach to ours and plugged exponential moving average into the attention mechanism to improve position-wise local dependency \cite{ma2022mega}. Our CES mechanism only considered a position-aware but representation-agnostic dependency which is completely different from the attention mechanism but matches the performance of the transformer.

\noindent \textbf{Other attention free models}. MLP-Mixer \cite{tolstikhin2021mlp}, and its variants proposed to replace the attention with MLPs in computer vision task \cite{touvron2022resmlp,yu2022s2,tatsunami2022raftmlp}. Another MLP-based model gMLP showed the potential of MLPs to model sequence dependency and achieved comparable results in pretraining and downstream NLP tasks \cite{liu2021pay}. The Attention Free Transformer (AFT) replaced the attention mechanism with an element-wise multiplication and avoided the quadratic computation burden of the attention matrix \cite{zhai2021attention}. Recurrent Memory Transformer (RMT) added a special cache token and utilized the recursive components to increase the context length in the transformer \cite{bulatov2022recurrent,bulatov2023scaling}. Receptance Weighted Key Value (RWKV) leveraged token shift for parallel training a simple linear RNN \cite{peng2023rwkv}. Our models do not conflict with those models in spirit. Our CES modules can be integrated into their models to improve their capabilities of sequence learning.

\section{Conclusion}

We proposed the ETSMLP model for long-range sequence modeling. Our approach began with a special SSM, namely ETS, and incorporated additional hyperparameters. Moreover, we proposed an exponential parameterization and a constraint function essential for stable training. The experimental results demonstrated the effectiveness of the ETSMLP in sequence learning. Our proposed module could become a plug-in module on other models to enhance their sequence learning capabilities. We hope our research could provide valuable insights into the application of SSMs and encourage further exploration in this area.

\section{Limitations}

Our approach focuses on evaluating datasets containing fewer than 100,000 samples, where the influence of prior knowledge on performance is substantial. In the next phase, we aim to conduct experiments on pretraining. The considerable disparity between pretraining and training from scratch requires meticulous adjustment of exponential smoothing and ingenious design of the architecture, something like Mamba \cite{gu2023mamba}.

Another limitation of our approach lies in the empirical design of the constraint function. This arises from the potential for lambda to surpass the precision of 32-bit floating point numbers if its training range is not restricted, and leads to NaN results during backpropagation. We believe that a low granularity parameterization can effectively mitigate this concern. Our future work will prioritize establishing a smooth training process on the parameter space.

\section*{Acknowledgements}
This work was supported by the National Natural Science Foundation of China under grant number 62076009.

\bibliography{anthology,custom}

\appendix
\newtheorem*{proposition*}{Proposition}

\begin{table*}
	\centering
	{
		\begin{tabular}{|c|c|c|c|c|c|c|c|}
			\toprule
			Task& $D$ & $L$ & $H$ & $Norm$ &$LR$ &$WD$& $DP$\\
			\midrule
			\textit{ListOps}&160&12&160&layer&0.01&0.01&0.0\\
			\textit{Text}&160&4&160&layer&0.005&0.01&0.1\\
			\textit{Retrieval}&160&6&160&layer&0.005&0.01&0.1\\
			\textit{Image}&160&12&320&batch&0.01&0.01&0.0\\
			\textit{Pathfinder}&128&6&256&batch&0.01&0.01&0.0\\
			\textit{Path-X}&128&6&256&batch&0.05&0.01&0.0\\
			\textit{CoLA}&512&3&512&layer&1e-5&0.1&0.1\\
			\textit{SST-2}&512&12&512&layer&1e-5&0.1&0.1\\
			\textit{iMDB}&512&4&512&layer&1e-5&0.1&0.3\\
			\textit{QQP}&512&6&512&layer&1e-5&0.1&0.1\\
			\textit{MRPC}&512&6&512&layer&1e-5&0.1&0.1\\
			\textit{MNLI}&512&6&512&layer&1e-5&0.1&0.1\\
			\textit{QNLI}&512&6&512&layer&1e-5&0.1&0.1\\
			\bottomrule
	\end{tabular}}
	\caption{The hyperparameters of the ESMLP on LRA and NLU tasks. $D$ is the embedding size, $H$ is the hidden features, and $L$ is the number of layers. $LR$ is the learning rate, $WD$ is weight decay and $DP$ is dropout. $BN$ and $LN$ in the column $Norm$ refer to Batch Normalization and Layer Normalization. For NLU tasks, the small and base models have different model scales.}
	\label{tab:hyper}
\end{table*}

\section{Proof of Proposition 1}
\label{app:proof}

We first restate Proposition 1 for the sake of convenience.
\begin{proposition*}
Let $\lambda \in \mathbb{C}$  be within the interior of the hollow unit disc $D^\circ(0,1)=\{z||z|<1\}/\{(0,0)\}$. We define $\lambda'=\log\log \lambda$ which substitutes $\lambda$ in the equation (\ref{eq:cest}). If the gradients of $y_t$ satisfy $\sum_{0}^{L} \frac{\mathrm{d} L}{\mathrm{d} y_t} y_t < \infty$, the gradients of the real and imaginary parts of $\lambda'$ are bounded for all $\lambda \in D^\circ(0,1)$.
\end{proposition*}

\begin{proof}
Let the real and imaginary parts of $\lambda'$ be $a$ and $b$, thus $\lambda'=a+bi$.

We define intermediate variables $y'_t\in \mathbb{C}$ for $t$ in range $\{0,1,\dots,L\}$ and  
$$y'_t=\sum_{j=0}^{t-1}(\lambda^\alpha)^j(1-\lambda^\alpha)\beta x_{t-j}.$$
Compared with the equation (\ref{eq:cest}), we find that $y_t=\Re(y'_t)$. 

We have that  if $\sum_{0}^{L} \frac{\mathrm{d} L}{\mathrm{d} y_t} y_t < \infty$, then $\sum_{0}^{L} \frac{\mathrm{d} L}{\mathrm{d} y_t} \Im(y'_t) < \infty$. This is proved by the following: 

{\small
\begin{align*}
	\tiny
	\frac{\Re(y'_t)}{\Im(y'_t)}&=\frac{\sum_{j=0}^{t}\Re(\beta (1-\lambda^\alpha)\lambda^{\alpha*j})x_{t-i}}{\sum_{j=0}^{t}\Im(\beta (1-\lambda^\alpha)\lambda^{\alpha*j})x_{t-i}}\\
	&=\frac{\sum_{j=0}^{t}\Re((\beta_1+i\beta_2)((\lambda^\alpha)^j-(\lambda^\alpha)^{j+1}))x_{t-i}}{\sum_{j=0}^{t}\Im((\beta_1+i\beta_2)((\lambda^\alpha)^j-(\lambda^\alpha)^{j+1}))x_{t-i}}\\
	&=\frac{\sum_{j=0}^{t}|\lambda^\alpha|^j(\beta_1(\cos (j\theta)-|\lambda^\alpha|\cos((j+1)\theta))-}{\sum_{j=0}^{t}|\lambda^\alpha|^j(\beta_2(\cos (j\theta)-|\lambda^\alpha|\cos((j+1)\theta))+} \\
        &\quad \frac{\beta2(\sin(j\theta)-|\lambda^\alpha|\sin((j+1)\theta)))x_{t-i}}{\beta2(\sin(j\theta)-|\lambda^\alpha|\sin((j+1)\theta)))x_{t-i}}.
\end{align*}
}

\noindent It is obvious that the ratio $\frac{\Re(y'_t)}{\Im(y'_t)}$ is bounded for the finite summation $\sum_{0}^{L} \frac{\mathrm{d} L}{\mathrm{d} y_t} \Im(y'_t)$ is bounded too.

To compute the gradients of $a$ and $b$, we consider the gradients of $\lambda'$ for $y'_t$. The function of $y'_t$ is holomorphic function for $\lambda'$ thus the gradients is:

{\small
\begin{align*}
\frac{\mathrm{d} y'_t}{\mathrm{d} \lambda'}&=\sum_{j=0}^{t} \beta x_{t-j} (\alpha j \lambda ^{\alpha j -1}-\alpha (j+1) \lambda^{\alpha (j+1)-1}) e^{e^{\lambda'}}e^{\lambda'}\\
&=\sum_{j=0}^{t} \beta x_{t-j} \alpha \lambda ^{\alpha j}( j - (j+1) \lambda^{\alpha}) \log(\lambda)\\
&=\sum_{j=0}^{t} \beta x_{t-j} \alpha \lambda ^{\alpha j}( j - (j+1) \lambda^{\alpha}) \log(\lambda).
\end{align*}
}

\noindent As $y'_t$ is a holomorphic function for $\lambda'$ and $y_t=\Re(y'_t)$, we thus have:
$$\frac{\mathrm{d} y'_t}{\mathrm{d} \lambda'}=\frac{\mathrm{d} y_t}{\mathrm{d} a}-\frac{\mathrm{d} y_t}{\mathrm{d} b}i.$$

\noindent Thus, the gradients of $a$ and $b$ is computed by the chain rule in the following:

{\small
\begin{align*}
\frac{\mathrm{d} L}{\mathrm{d} a}&= \sum_{t=1}^{L} \frac{\mathrm{d} L}{\mathrm{d} y_t} \frac{\mathrm{d} y_t}{\mathrm{d} a}\\
&=\sum_{t=1}^{L} \frac{\mathrm{d} L}{\mathrm{d} y_t} \Re(\sum_{j=0}^{t} \beta x_{t-j} \alpha \lambda ^{\alpha j}( j - (j+1) \lambda^{\alpha}) \log(\lambda));
\end{align*}
\begin{align*}
\frac{\mathrm{d} L}{\mathrm{d} b}&= \sum_{t=1}^{L} \frac{\mathrm{d} L}{\mathrm{d} y_t} \frac{\mathrm{d} y_t}{\mathrm{d} b}\\
&=-\sum_{t=1}^{L} \frac{\mathrm{d} L}{\mathrm{d} y_t}\Im(\sum_{j=0}^{t} \beta x_{t-j} \alpha \lambda ^{\alpha j}( j - (j+1) \lambda^{\alpha}) \log(\lambda)).
\end{align*}
}

\noindent Obviously, $\frac{\mathrm{d} L}{\mathrm{d} a}$ and $\frac{\mathrm{d} L}{\mathrm{d} b}$ are continue on all $\lambda \in D^\circ(0,1)$. 

We consider the boundary of $D^\circ(0,1)$. We first take a look at the zero point. For $\lim_{\lambda\rightarrow \infty} \lambda ^{\alpha j} \log(\lambda) =0 $, we can easily compute the limitation :

\begin{align*}
\tiny
\lim_{\lambda \rightarrow 0}\frac{\mathrm{d} L}{\mathrm{d} a}&=\sum_{t=1}^{L} \frac{\mathrm{d} L}{\mathrm{d} y_t} \Re(\sum_{j=0}^{t} \beta x_{t-j} \alpha j *(\lambda ^{\alpha j} \log(\lambda)))\\
&=0.
\end{align*}

\noindent For the point $\lambda_0$ on the unit cycle, we can find a constant $C(\lambda_0)\in D(0,2L)$ which satisfies  $\Re(\beta \alpha \lambda_0 ^{\alpha j}(j - (j+1) \lambda_0^{\alpha}))\le \Re(\beta \alpha \lambda_0 ^{\alpha j}*C(\lambda_0))$ for all $j\in \{1,2,\dots,L\}$. Therefore,

\begin{align*}
\tiny
\frac{\mathrm{d} L}{\mathrm{d} a}&= \sum_{t=1}^{L} \frac{\mathrm{d} L}{\mathrm{d} y_t} \frac{\mathrm{d} y_t}{\mathrm{d} a}\\
&\le \sum_{t=1}^{L} \frac{\mathrm{d} L}{\mathrm{d} y_t} \Re(\sum_{j=0}^{t} \beta x_{t-j} \alpha \lambda ^{\alpha j} C(\lambda_0) \log(\lambda))\\
&= \sum_{t=1}^{L} \frac{\mathrm{d} L}{\mathrm{d} y_t} \Re( \frac{y'_t}{1-\lambda^\alpha} \alpha C(\lambda_0) \log(\lambda)).
\end{align*}

\noindent We know that $\frac{\log(\lambda)\alpha}{1-\lambda^\alpha}$ is finite on the unit cycle except for $1$ and the $\lim_{\lambda\rightarrow 1}\frac{\log(\lambda)\alpha}{1-\lambda^\alpha}=-1$ is also finite. As a result, we can find a constant $D$ which satisfies $\Re(y'_t \frac{\alpha \log( \lambda_0)}{1-\lambda_0^\alpha}  C(\lambda_0))\le \Re(y'_t) D$ for all $t=\{1,\dots,L\}$ and $\lambda_0$. Thus, for all $\lambda_0$ on the unit cycle we have:

\begin{align*}
\tiny
\frac{\mathrm{d} L}{\mathrm{d} a}&\le \sum_{t=1}^{L} \frac{\mathrm{d} L}{\mathrm{d} y_t} \Re(y'_t) D<\infty .
\end{align*}

\noindent Similarly, we have:

\begin{align*}
\tiny
\frac{\mathrm{d} L}{\mathrm{d} a}&\le \sum_{t=1}^{L} \frac{\mathrm{d} L}{\mathrm{d} y_t} \Im(y'_t) D<\infty .
\end{align*}

\noindent As the function $\frac{\mathrm{d} L}{\mathrm{d} a}$ and $\frac{\mathrm{d} L}{\mathrm{d} b}$ are continues and the boundaries are finite, by the boundedness theorem, we conclude that the gradients of $a$ and $b$ are bounded for all $\lambda \in D^\circ(0,1)$.
\end{proof}

\section{Experimental setup}
\label{app:datasets}

\noindent\textbf{Architecture}. We present an overview of our architecture in Figure \ref{Fig:Architecture}. The ETSMLP architecture contains $L$ blocks and each block contains a normalization, skip connection, and an MLP plus CES. We use the ReLU activation function in MLPs. For ETSMLP-Gate architecture, an extra gate mechanism is added parallel to the original architecture. For the sake of performance, we add extra normalization like LRU \cite{orvieto2023resurrecting}. We use bidirectional models for all datasets.

\noindent \textbf{Experimental details}. We use the Adam optimizer \cite{kingma2017adam}, with the hyperparameter $\beta_1=0.9$, $\beta_2=0.98$.  We use warmup for the learning rate $LR$ that we start from a value of $10^{-7}$ and increase the learning rate linearly up a specified value for the first 10\% of training. Then a linear annealing schedule is conducted for the rest of the training. All experiments except for Path-X were carried out on an NVIDIA GeForce GTX 2080 GPU, while Path-X requires 8 NVIDIA GeForce GTX 2080 GPUs.

\noindent \textbf{Hyperparameters}. We follow the general optimization approach used by Mega \cite{ma2022mega}. Table \ref{tab:hyper} presents the main hyperparameters for each experiment. For all the experiments, we tune the embedding size $D$, the number of layers $L$, and the hidden features $H$. We also tune the learning rate $LR$ and weight decay $WD$ for all the datasets. Besides, the $max_\lambda$ of the constraint function is all set to $0.9999$.

\end{document}